\documentclass[journal]{IEEEtran}
\usepackage{caption}
\usepackage{subcaption}
\usepackage{lipsum}
\usepackage{url}
\usepackage{color}

\newcommand{\overbar}[1]{\mkern 1.5mu\overline{\mkern-1.5mu#1\mkern-1.5mu}\mkern 1.5mu}
\usepackage{graphics}
\usepackage{amsmath,amsfonts, amssymb}
\usepackage{mathtools}
\usepackage{epstopdf}
\usepackage{amsthm}
\usepackage{bbm}
\usepackage{algorithm}
\usepackage[noend]{algpseudocode}
\usepackage{dsfont}
\usepackage{cite}
\newtheorem{assumption}{Assumption}

\newtheorem{theorem}{Theorem}

\begin{document}
\pagestyle{empty}

\title{Hierarchical Federated Learning Across Heterogeneous Cellular Networks}

\author{Mehdi~Salehi Heydar Abad and
        Emre~Ozfatura and
        Deniz~Gunduz and 
        Ozgur~Ercetin
\thanks{Emre Ozfatura and Deniz G{\"u}nd{\"u}z are with Information Processing and Communications Lab, Department of Electrical and Electronic Engineering,
Imperial College London Email: \{m.ozfatura, d.gunduz\} @imperial.ac.uk.}
\thanks{Mehdi Salehi Heydar Abad and Ozgur Ercetin are with the Faculty of Engineering and Natural Sciences, Sabanci University  Email: \{mehdis, oercetin\} @sabanciuniv.edu} 
\thanks{This work was supported in part by the Marie Sklodowska-Curie Action SCAVENGE (grant agreement no. 675891), and by the European Research Council (ERC) Starting Grant BEACON (grant agreement no. 677854).}
%\thanks{$^{\dagger}$ Equally contributed.}

}
\maketitle
\thispagestyle{empty}
\begin{abstract}
We study collaborative machine learning (ML) across wireless devices, each with its own local dataset. Offloading these datasets to a cloud or an edge server to implement powerful ML solutions is often not feasible due to latency, bandwidth and privacy constraints. Instead, we consider federated edge learning (FEEL), where the devices share local updates on the model parameters rather than their datasets. We consider a heterogeneous cellular network (HCN), where small cell base stations (SBSs) orchestrate FL among the mobile users (MUs) within their cells, and periodically exchange model updates with the macro base station (MBS) for global consensus. We employ gradient sparsification and periodic averaging to increase the communication efficiency of this hierarchical federated learning (FL) framework. We then show using CIFAR-10 dataset that the proposed hierarchical learning solution can significantly reduce the communication latency without sacrificing the model accuracy.  

%\indent In 

 %In this  paper, to fill this gap, we introduce a hierarchical distributed edge learning strategy which utilizes in together the local stochastic gradient descent (SGD) framework  and the frequency reuse strategy to speed up the learning process.  In this approach, computing edge nodes are divided into clusters and one node in each cluster act as a parameter server (PS) to gather the local computations. This way multiple nodes , equal to the number of clusters, can communicate at the same time under the assumption that interference between the clusters is negligible. Further, the nodes assigned as the PS can exchange their local models periodically, resulting a hierarchical model. In this paper, using CIFAR-10 dataset and ResNet-18 topology, we successfully show that proposed strategy converges much faster compared to the conventional schemes.
\end{abstract}
\maketitle

\section{Introduction}
Vast amount of data is generated today by mobile devices, from smart phones to  autonomous vehicles, drones, and various Internet-of-things (IoT) devices, such as wearable sensors, smart meters, and surveillance cameras. Machine learning (ML) is key to exploit these massive datasets to make intelligent inferences and predictions. Most ML solutions are centralized; that is, they assume that all the data collected from numerous devices in a distributed manner is available at a central server, where a powerful model is trained on the data. However, offloading these huge datasets to an edge or cloud server over wireless links is often not feasible due to latency and bandwidth constraints. Moreover, in many applications dataset reveal sensitive personal information about their owners, which adds privacy as another concern against offloading data to a centralized server. A recently proposed alternative approach is federated edge learning (FEEL) \cite{ML_overair2, ML_overair3, ML_overair4, int_edge1}, which enables ML at the network edge without offloading any data.

{\em Federated learning (FL) }is a collaborative ML framework  \cite{fedlearn1, fedlearn2}, where random subsets of devices are selected in an offline manner to update model parameters based on locally available data. Local models are periodically averaged and exchanged among participating devices. This can either be done with the help of a parameter server, which collects the local model updates and shares the updated global model with the devices; or, in a fully distributed manner, where the devices taking part in the collaborative training process seek consensus on the global model through device-to-device communications. 

Although the communication bottleneck of FL has been acknowledged in the ML literature, and various communication-efficient distributed learning techniques have been introduced, implementation of these techniques on wireless networks, particularly in heterogeneous cellular networks (HCNs), and the successful orchestration of the large scale learning problem have not been fully addressed. To this end, there are some very recent works that focus on the distributed machine learning framework with a particular focus on wireless communications \cite{ML_overair1,ML_overair2,ML_overair3,ML_overair4,ML_overair5,ML_overair6,ML_overair7,ML_overair8, FL_wireless1, FL_wireless2}. Most of these works propose  new communication-efficient learning strategies, specific to wireless networks,  which is called the {\em over-the-air aggregation} \cite{ML_overair1,ML_overair2,ML_overair3,ML_overair4,ML_overair5,ML_overair6,ML_overair7,ML_overair8}. In this approach, mobile computing nodes are synchronised for concurrent transmission of their local gradient computations or model updates over the wireless channel, and the parameter server receives the noisy version of the gradient sum via utilizing the superposition property of the wireless channel. Although, over-the-air aggregation is a promising solution to mitigate the communication bottleneck in the future communication networks, it imposes stringent synchronization requirements and very accurate power alignment, or, alternatively, the use of a very large number of antennas \cite{ML_overair8}.

\indent In this paper, we focus on FEEL across HCNs, and introduce a communication-efficient hierarchical ML framework. In this framework mobile users (MUs) with local datasets are clustered around small-cell base stations (SBSs) to perform distributed stochastic gradient descent (SGD) with decentralized datasets, and these SBSs communicate with a macro-cell base station (MBS) periodically to seek  consensus on the shared model of the corresponding ML problem. In order to further reduce the communication latency of this hierarchical framework, we utilize  gradient sparsification, and introduce an optimal resource allocation scheme for  synchronous gradient updates.

Distributed hierarchical SGD framework has been recently studied in \cite{sgd_local4, sgd_local5}, and hierarchical FL is considered in \cite{fedlearn8}. However, only periodic averaging strategy is employed in these works, and the wireless nature of the communication medium is not taken into account. Our contributions in this paper can be summarized as follows:
 \begin{itemize}
\item We introduce a hierarchical FL framework for HCNs and provide a holistic approach for the communication latency with a rigorous end-to-end latency analysis.  
\item We  employ communication efficient distributed learning techniques, in particular, sparsification and periodic averaging, jointly, and design a resource allocation strategy  to minimize the end-to-end latency.
\item Finally, focusing on the distributed image classification problem using popular dataset CIFAR-10, we demonstrate that, with the proposed approach, communication latency in a large scale FEEL framework, implemented over HCNs, can be reduced dramatically without sacrificing the accuracy much.  
\end{itemize}

%%%%%%%%%%%%%%%%%%%%%%%%%
\section{System Model}\label{s:System_Model}
%%%%%%%%%%%%%%%%%%%%%%%%%

Consider a HCN with a single MBS and $N$ SBSs. In this cellular setting, $K$ MUs collaborate to jointly solve an optimization problem of the form 
\begin{align}
    \min_{\mathbf{w}\in \mathbb{R}^Q} \frac{1}{L} \sum^{L}_{i=1} f_i(\mathbf{w}),
\end{align}
where $\mathbf{w}$ is a vector of size $Q\times 1$ denoting the parameters of the model to be learned and $f_i$ is the training loss associated with the $i$th training data sample. We assume that MU $k$ has a training data set $\mathcal{D}^k$ that is not shared with any other MUs due to bandwidth and privacy constraints. Note that calculating the loss over all of the dataset is time consuming and in some cases not feasible since it cannot fit in the memory. Thus, we employ minibatch SGD in which MU $k$ uses a subset $\mathcal{I}^k\in \mathds{D}^k$ of its dataset to calculate the loss. We assume that the batch size $|\mathcal{I}^{k}|=I$ for all $k=1,\ldots,K$.
In distributed learning, each MU calculates the loss gradients with respect to its own data set. Then, the gradients are shared with other MUs through either  peer-to-peer links, or using a central entity (MBS in this work). The MBS collects the gradients, aggregates them by taking the average, and eventually transmits the average gradient to all the MUs \footnote{Note that the MBS can update the model itself and transmit the updated model instead of the average gradients, thus avoiding replicating the update at $K$ MUs. However, it is possible to apply sparsification on the average gradient to improve latency in this manner.}. Each MU upon receiving the average gradient applies a  gradient descend step as follows:
\begin{equation}
    \mathbf{w}_{t+1} = \mathbf{w}_{t} - \eta_t  \left[ \frac{1}{K}\sum^{K}_{k=1} \frac{1}{|\mathcal{I}^{k}|}\sum_{i\in\mathcal{I}^{k}}\nabla f_i(\mathbf{w}_t) \right],
\end{equation}
where $\mathcal{I}^{k}\subseteq \mathcal{D}^k$ is the mini-batch randomly chosen from the data samples of MU $k$, and $\eta_t$ is the learning rate. The generic federated learning (FL) algorithm is described in Algorithm \ref{alg:dist}. The Algorithm \ref{alg:dist} is synchronous in the sense that MBS waits until the gradients from all the MUs are received.

\begin{algorithm}
\caption{Federated learning (FL) algorithm}\label{alg:dist}
\begin{algorithmic}[1]
\State Initialize $\mathcal{D}_k$, learning rate $\eta$, $\mathbf{w}_k=\mathbf{w}_0$
\State Each MU $k$ does the following
\For{$t=1,2,\ldots$}
\State Randomly select a mini-batch $\mathcal{I}_k\subseteq\mathcal{D}_k$
\State Calculate $\mathbf{g}_{k,t}=\frac{1}{|\mathcal{I}^{k}|}\sum_{i\in\mathcal{I}^{k}}\nabla f_i(\mathbf{w}_t)$
\State Transmit $\mathbf{g}_{k,t}$ to MBS \label{alg:Tx}
\State Receive $\mathbf{g}_{t} = \frac{1}{K}\sum_k \mathbf{g}_{k,t}$ \label{alg:Rx}
\State $\mathbf{w}_{t+1} = \mathbf{w}_t - \eta_t \mathbf{g}_{t}$
\EndFor
\end{algorithmic}
\end{algorithm}

% Consider a collaborative machine learning problem with a single parameter server (PS) and $K$ workers. We assume that each worker is comparable to the PS in terms of computation and storage capabilities and hence can itself act as a parameter server. The PS and workers $k=1,\ldots,K$ aim at training a shared model with parameters $\mathbf{w}$. Each worker has only a fraction of the global data set denoted by $\mathcal{D}_k$.

% Training the common model is realized by minimizing a loss function over the local data set by the workers. Denote the loss function of the model with parameters $\mathbf{w}$ and $i$-th labeled data, $(x_i,y_i)$, by $f(w;x_i,y_i)$. Then the function to be minimized by worker $k$ becomes:
% \begin{align}
%     F_k(\mathbf{w}) = \frac{1}{|\mathcal{D}_k|}\sum_{(x_i,y_i)\in \mathcal{D}_k }f(\mathbf{w};x_i,y_i)
% \end{align}
% Denote by $\mathbf{w}_k=[w_{k,1},w_{k,2},\ldots,w_{k,q}]^{T}$ the local parameters at worker $k$.
Note that transmitting local gradients and receiving the average gradients (i.e., lines \ref{alg:Tx} and \ref{alg:Rx}) introduce latency to the training time specially considering the deep neural networks with tens of millions of parameters. Hence, an efficient communication protocol is required for this purpose considering the synchronous nature of Algorithm \ref{alg:dist}. 
We assume that the bandwith available for communication is $B$ Hz. We employ an orthogonal access scheme with OFDM, and assign distinct sub-carriers to MUs . Denote by $M=\frac{B}{B_0}$ the number of sub-carriers, where $B_0$ is sub-carrier spacing. We denote the channel gain between MU $k$ and MBS on sub-carrier $m$ by $\gamma_{k,m} = |h_{k,m}|^2$, where $h_{k,m} $ is the complex channel coefficient. The distance of MU $k$ to MBS is denoted by $d_k$, and the path loss exponent by $\alpha$.

\subsection{Uplink Latency}
For the latency analysis, we consider the fixed-rate transmission policy with sub-optimal power allocation introduced in \cite{goldsmith},  which is simple to implement and performs closely to the optimal water-filling power allocation algorithm. The power allocation policy is truncated channel inversion, which only allocates power if the channel gain is above a threshold, otherwise does not use that subcarrier.

Let $p_{k,m}$ denote the power allocated to sub-carrier $m$ by MU $k$ based on the observed channel gain $\gamma_{k,m}$, and  let $\mathcal{M}_k$ be the set of uplink (UL) sub-carriers assigned to MU $k$. We should satisfy the average power constraint:
\begin{align}
    \mathbb{E}\bigg[ \sum_{m\in\mathcal{M}_k} p_{k,m} \bigg] \leq P_{max},
\end{align}
where the expectation is with respect to the probability density function (pdf) of the channel gain, $f(\gamma_{k,m})$.
Since the channel gain is i.i.d over sub-carriers the power constraint becomes
\begin{align}
    \mathbb{E}\left[  p_{k,m} \right] \leq \frac{P_{max}}{|\mathcal{M}_k|},\ m\in \mathcal{M}_k . \label{eq:power-const1}
\end{align}
According to the truncated channel inversion policy, the allocated power by MU $k$ on sub-carrier $m\in\mathcal{M}_k$ becomes
\begin{equation}
p_{k,m} = \begin{dcases}
\frac{\rho_{k,m}}{\tilde{\gamma}_{k,m}}, & \gamma_{k,m}\geq \gamma^{th}_{k,m}\\
0 & \gamma_{k,m}< \gamma^{th}_{k,m}
\end{dcases},
\label{eq:power-allocation1}
\end{equation}
where $\rho_{k,m}$ ensures that the power constraint in  \eqref{eq:power-const1} is met and,

\begin{equation}
\tilde{\gamma}_{k,m}=\frac{\gamma_{k,m}}{N_0 B_0 d_{k}^{\alpha}}
\end{equation}
is the normalized channel gain and 
 $N_0 B_0$ is the AWGN noise power on a single sub-carrier. The average power constraint in \eqref{eq:power-const1} results in \cite{goldsmith}:
\begin{align}
    \rho_{k,m}(\gamma^{th}_{k,m}) = \frac{P_{max}}{|\mathcal{M}_k| N_0 B_0 d_k^{\alpha} \overbar{[\frac{1}{\gamma_{k,m}}]}_{\gamma^{th}_{k,m}}},
\end{align}
where
\begin{align}
    \overbar{\left[\frac{1}{\gamma_{k,m}}\right]}_{\gamma^{th}_{k,m}} \overset{\Delta}{=} \int^{\infty}_{\gamma^{th}_{k,m}} \frac{f(\gamma_{k,m})}{\gamma_{k,m}} d\gamma_{k,m}.
\end{align}
Rather than Shannon capacity, we consider a practical approach where the bits are modulated using $M$-ary QAM. For a given target {\em bit error rate} (BER) requirement, the instantaneous UL rate of MU $k$ to the MBS on sub-carrier $m$ becomes \cite{goldsmith}:
\begin{align}
    U_{k,m} = B_0 \log_2\left( 1 + \frac{1.5}{-\ln(5 BER)}\frac{p_{k,m} \gamma_{k,m}}{N_0B_0 d_{k}^{\alpha}}\right). \label{eq: inst-rate1}
\end{align}
By substituting \eqref{eq:power-allocation1} into \eqref{eq: inst-rate1}, the instantaneous rate becomes:

\begin{align}
U_{k,m} = 
B_0 \log_2\left( 1 + \frac{1.5\rho_{k,m}(\gamma^{th}_{k,m})}{-\ln(5 BER)} \right)\mathds{1}_{\gamma_{k,m} \geq \gamma^{th}_{k,m}}, 
\end{align}
where $\mathds{1}_{x}=1$ if the argument $x$ is true and $0$ otherwise.
For the maximum expected transmission rate, we have

\begin{align}\label{eq:ULrate}
    \bar{U}_{k,m} = \max_{\gamma^{th}_{k,m}}  B_0 \log_2\left( 1 + \frac{1.5\rho_{k,m}(\gamma^{th}_{k,m})}{-\ln(5 BER)}\right)\mathbb{P}(\gamma_{k,m} \geq \gamma^{th}_{k,m})
\end{align}

%The cluster head $n$ collects the gradients from all its members $k\in\mathcal{C}_n$ and then computes the average gradient. Hence, a given cluster head can only compute the average gradient when all its members finish uploading their gradients. 
The average UL rate of MU $k$ for gradient aggregation:
\begin{align}
    \bar{U}_k = \sum_{m\in \mathcal{M}_k}\bar{U}_{k,m} \label{eq:UL-Worker-CH}
\end{align}
Each MU uses $\hat{Q}$ bits to quantize each element of its gradient vector. Since the model has $Q$ parameters,  each MU needs to send $Q\cdot \hat{Q}$ bits in total to the MBS at each iteration. To minimize the latency of uploading the gradients to the MBS, we should allocate the sub-carriers so that the minimum average UL rate among MUs is maximized. Hence, we perform the following optimization problem:

\begin{align}\label{opt-ul}
    &\max_{\mathcal{M}_k, \gamma^{th}_{k,m}}\min\left(\sum_{m\in \mathcal{M}_k}\bar{U}_{k,m}\right)^K_{k=1} \nonumber\\
    &\text{s.t.}  \bigcup^{K}_{k=1} \mathcal{M}_k = \mathcal{M} 
\end{align}
For the given solution of \eqref{opt-ul}, i.e., the optimal sub-carrier allocation,  $\left\{\mathcal{M}^{\star}_{k}\right\}^K_{k=1}$, the uplink latency of MU $k$ on average is
\begin{align}
    T^{UL}_{k} = \frac{Q\Hat{Q}}{\bar{U}^\star_k},
\end{align}
where $\bar{U}^\star_k = \sum_{m\in \mathcal{M}^{\star}_k}\bar{U}_{k,m}$. Accordingly, the uplink latency in aggregating the gradients of MUs is equal to 
\begin{equation}
T^{U}=\max_k(T^{U}_{k}).
\end{equation}

\subsection{Downlink Latency}
After all the MUs transmit their gradients to the MBS, the average gradient $\mathbf{g} = \sum_{k=1}^{K} \mathbf{g}_k$ is calculated and MBS is required to transmit the average value back to the MUs. However, since the all workers share a common message, we employ a broadcast policy for this case. We assume that the MBS employ a rateless coding scheme that is adapted to the worst instantaneous signal-to-noise ratio (SNR) on each subcarrier. We assume that the MBS allcoates its available power uniformely over all subcarriers.
Specifically, let $SNR_{k,m}(t)$ be the SNR of worker $k$ on subcarrier $m$. Then, the instantaneous broadcast rate on subcarrier $m$ becomes:

\begin{align}
    R_m(t) = \min_{k} B_0 \log_2( 1 + SNR_{k,m}(t)),
\end{align}
where, 
\begin{align}
    SNR_{k,m}(t) = \frac{P_{max} \gamma_{k,m}(t)}{M N_0 B_0 d^{\alpha}_k}
\end{align}

The broadcast will end when all $Q\cdot\hat{Q}$ parameters are received by the workers. The broadcast latency, $T^{DL}$,  can be computed as follows:
\begin{align}
    T^{DL} = \mathds{E}\min \{t | T_s\sum^{t}_{\tau=1}\sum^{M}_{m=1}R_{m}(\tau)\geq Q\hat{Q} \} ,
\end{align}
where the expectation is with respect to the PDF of the channel gain. Per iteration, the end-to-end latency of the FL protocol is given by $T^{FL} = T^{UL} + T^{DL}$

\subsection{Sub-carier Allocation Policy}

\begin{algorithm}
\caption{Sub-carrier allocation}\label{alg:sub_alloc}
\begin{algorithmic}[1]
\State Initialize (number of sub-carriers $M^{U}$, number of MUs $K$)
\State Initialize $M^{U}_k=1$ for all $k$\label{alg:sub_allocL2}
\While{$\sum_{k=1}^{K}M^{U}_k < M^{U}$}
\State optimize $\gamma^{th}_{k,m}$ in \eqref{eq:ULrate}
\State $k^{\star} = \arg\min \bar{U}_{k}$
\State $M_{k^{\star}}\leftarrow M_{k^{\star}} + 1$
\EndWhile
\end{algorithmic}
\end{algorithm}

The optimal sub-carrier allocation problem in \eqref{opt-ul} is presented in Algorithm \ref{alg:sub_alloc}. It starts by assigning a single sub-carrier for each MU. Then with a single sub-carrier, each MU $k$ optimizes the threshold in \eqref{eq:ULrate}. Then the algorithm looks for a MU with minimum average UL rate, i.e.,  $k^{\star} = \arg\min \bar{U}_{k}$, and allocates a single carrier to that MU. Then MU $k^{\star}$ updates its threshold and $\bar{U}^\star_k$ value. This procedure continues until all available sub-carriers are allocated. The following theorem establishes the optimality of the proposed policy.

\begin{theorem}
The sub-carrier allocation policy in Algorithm \ref{alg:sub_alloc} is optimal.
\end{theorem}
\begin{proof}
The proof is by induction. Let the number of sub-carriers be $M^{U}= K + 1$. Then, a single sub-carrier is allocated to each MU first (line \ref{alg:sub_allocL2}), since otherwise there will be a single MU with a rate of zero. The optimal choice is to allocate the remaining single sub-carrier to $k^{\star} = \arg\min \bar{U}_{k}$. To see why, let $\bar{U}_{k^{\star}(m_k)}$ denote the rate of MU $k^{\star}$ when $m_k$ sub-carriers are allocated. Now consider an alternative policy that allocates the remaining sub-carrier to a different MU $k\neq k^{\star}$. Denote the rates achieved under this alternative policy by $\bar{U}^{\ast}_{k}$. It is obvious that $\bar{U}_{k^\star}(2)>\bar{U}_{k^\star}(1) = \min_{k} \bar{U}^{\ast}_{k}$. Thus, the optimal policy allocates the remaining sub-carrier to MU $k^{\star}$. Now, assume that Algorithm \ref{alg:sub_alloc} allocates $M^{U}=K + m$ sub-carrier optimally. We need to prove the optimality of the algorithm for $M^{U}=K + m + 1$. Consider that $K + m$ sub-carriers are allocated and we need to allocate the last sub-carrier. The last sub-carrier is allocated to $k^{\star}$ so that number of its sub-carrier becomes $m_k+1$. The alternative policy allocates the last sub-carrier to another MU $k\neq k^{\star}$. It is clear that $\bar{U}_{k^\star}(m_k+1)>\bar{U}_{k^\star}(m_k) = \min_{k} \bar{U}^{\ast}_{k}$. Hence, the alternative policy is sub-optimal. This concludes the proof.
\end{proof}

\section{Distributed Hierarchical Federated Learning}
In centralized FL \cite{fedlearn6}, MUs transmit their computation results (local gradient estimates) to the parameter server (MBS) for aggregation at each iteration. However, in large scale networks, this centralized framework may result in high communication latency and thus increases the convergence time. To this end, we introduce {\em hierarchical federated learning}, in which multiple parameter servers  are employed to reduce the communication latency.
In the proposed hierarchical framework, MUs are clustered according to their locations. In each cluster a small cell base station (SBS) is tasked with being the parameter server. At each iteration, MUs in a cluster send their local gradient estimates to the SBS for aggregation, instead of the MBS. Then, the SBSs compute the average gradient estimates and transmit the results back to their associated MUs to update their model accordingly. 

In this framework, gradient communication is limited to clusters, particularly between the MUs and the corresponding SBSs. This not only reduces the communication distance, communication latency, but also allows the spatial reuse of the available communication resources. On the other hand, limiting the gradient communications within clusters may prevent convergence to a single parameter model (i.e., global consensus).  

To this end, we combine aforementioned intra-cluster gradient aggregation method with inter-cluster model averaging strategy, such that after every $H$ consecutive intra-cluster iterations, SBSs send their local model updates to the MBS to establish a global consensus. The overall procedure is illustrated in  Figure \ref{fig:HFL}.

\begin{figure*}[ht]
\begin{subfigure}{1\textwidth}
  \centering
  % include first image
  \includegraphics[width=.8\linewidth]{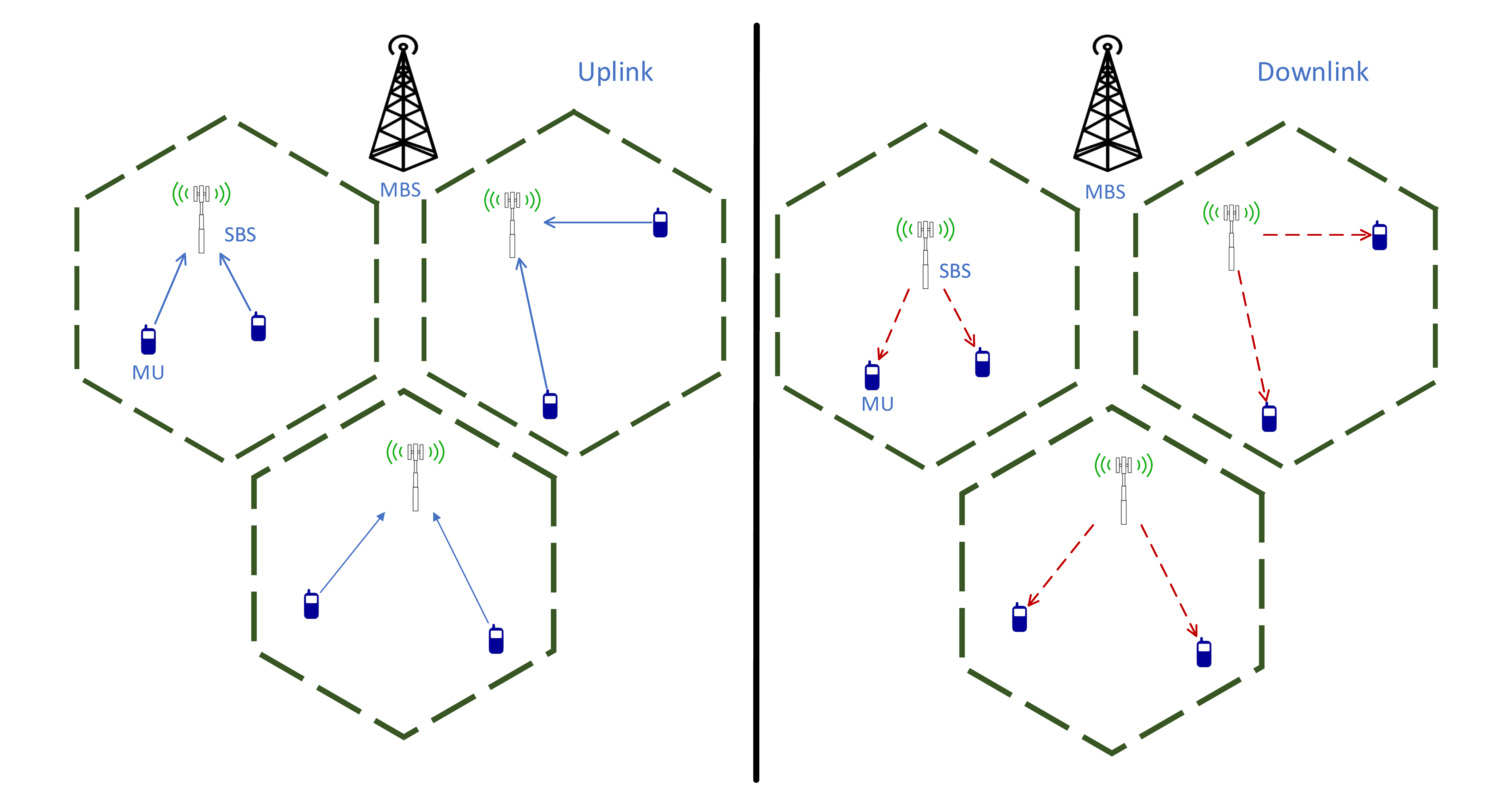}  
  \caption{Local gradient update}
  \label{fig:HFLl}
\end{subfigure}
\begin{subfigure}{1\textwidth}
  \centering
  % include second image
  \includegraphics[width=.8\linewidth]{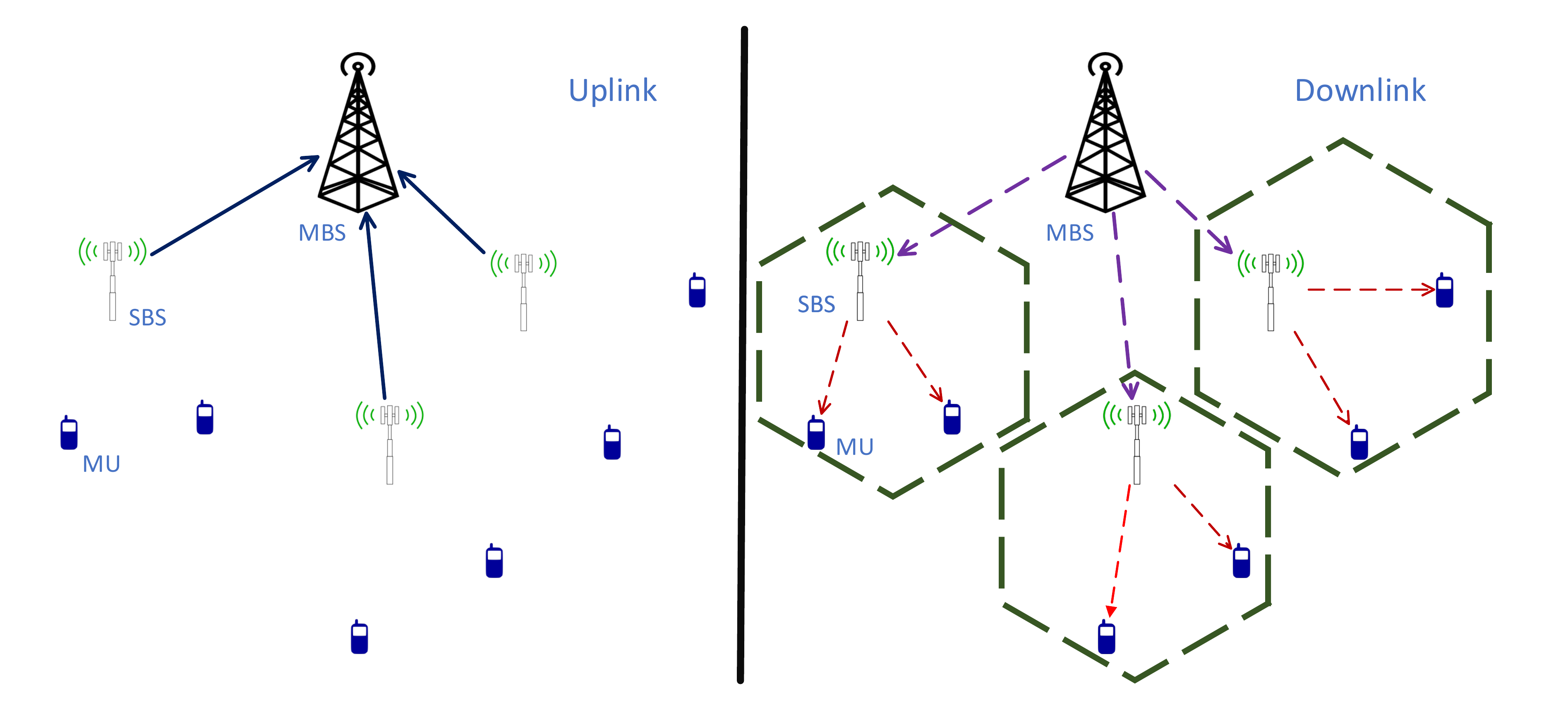}  
  \caption{Global model averaging}
  \label{fig:HFLlg}
\end{subfigure}
\caption{Hierarchical FL}
\label{fig:HFL}
\end{figure*}

Denote by $\mathcal{C}_n$ the set of MUs belonging to cluster $n=1,\ldots,N$, with $N$ being the number of clusters. During each consecutive $H$ intra-cluster iterations, the local gradient estimates of the MUs are aggregated within the clusters. For example, at iteration $t$ each MU $k$, $k\in\mathcal{C}_n$ for $n=1,\ldots,N$ computes the local gradient estimate, denoted by $\mathbf{g}_{n,k,t}=\frac{1}{|\mathcal{I}^{k}|}\sum_{i\in\mathcal{I}^{k}}\nabla f_i(\mathbf{w}_{n,t})$, and transmits it to the SBS in cluster $n$. Then, the SBS $n$ aggregates the gradients simply by taking the average,
\begin{align}
    \mathbf{g}_{n,t} = \frac{ \sum_{k\in\mathcal{C}_n} \mathbf{g}_{n,k,t}}{|\mathcal{C}_n|}. \label{eq:avg-grad-cluster}
\end{align}
This average is then sent back by the SBS to the MUs in its cluster, and the model at cluster $n=1,\ldots,N$ is updated as
\begin{align}
    \mathbf{w}_{n, t+1} =  \mathbf{w}_{n,t} - \eta_t \mathbf{g}_{n,t}
\end{align}
After $H$ iterations, all SBSs transmit their models to the MBS through UL fronthaul links. The MBS calculates the model average $\mathbf{w} = \sum^N_{n=1}\frac{\mathbf{w}_{n}}{N}$, and transmits it to the SBSs over the DL fronthaul links. Upon receiving the model update, the SBSs share it with the MUs in their cluster. Hence, after $H$ iterations all the MUs share a common model parameters, globally. The HFL algorithm is presented in Algorithm \ref{alg:dist-hier}.

\begin{algorithm}
\caption{Hierarchical federated learning (HFL)}\label{alg:dist-hier}
\begin{algorithmic}[1]
\State Initialize $\mathcal{D}_k$, learning rate $\eta$, $\mathbf{w}_k=\mathbf{w}_0$, $H$
\State Each MU $k$ does the following
\For{$t=1,2,\ldots$}
\State Randomly select a mini-batch $\mathcal{I}_k\subseteq\mathcal{D}_k$
\State Calculate $\mathbf{g}_{k,t}=\frac{1}{|\mathcal{I}^{k}|}\sum_{i\in\mathcal{I}^{k}}\nabla f_i(\mathbf{w}_t)$
\State Transmit $\mathbf{g}_{k,t}$ to SBS
\State Receive $\mathbf{g}_{n,t} = \frac{1}{K}\sum_k \mathbf{g}_{k,t}$ \label{alg:Rx}
\State $\mathbf{w}_{n, t+1} =  \mathbf{w}_{n,t} - \eta_t \mathbf{g}_{n,t}$
\If{$t$ is divisible by $H$}
\State $n$-th SBS, $n=1,\ldots,N$ send $\mathbf{w}_n$ to MBS
\State $\mathbf{w}\leftarrow \sum^{N}_{n=1}\frac{\mathbf{w}_n}{N}$
\State MBS transmit $\mathbf{w}$ to all SBSs
\State SBSs transmit $\mathbf{w}$ to their MUs
\State $\mathbf{w}_n=\mathbf{w}$
\EndIf
\EndFor
\end{algorithmic}
\end{algorithm}

\subsection{Communication Latency analysis}
In the hierarchical scheme, after clustering the MUs, clusters are colored so that any two clusters with the same color are separated by at least distance $D_{th}$ to minimize  interference between clusters. For simplicity, we assume that there is zero interference on receivers located beyond $D_{th}$. If $N_{c}$  colors are used in total, the available OFDM sub-carriers are divided into $N_{c}$ groups, and the sub-carriers in each group are allocated to clusters with a particular color. Consequently, in each cluster the number of available OFDM sub-carriers is proportional to  $1/N_{c}$. Before the delay analysis, we have the following assumptions regarding the location of the MUs.
\begin{assumption}
The MUs are uniformly distributed and each cluster contains equal number of MUs.
\end{assumption}
\begin{assumption}
The SBSs are located at the origin of the corresponding clusters.
\end{assumption}
In the local gradient update step of HFL (see Figure \ref{fig:HFLl}) communication latency analysis is similar to the FL. The only difference is the number of sub-carriers inside the clusters which is $M/N_{c}$. Moreover, the MUs transmit to the SBSs instead of MBS. Denote by $\bar{U}^{*n}_{k}$ the maximum average UL rate of MU $k\in\mathcal{C}_n$. The UL latency of gradient aggregation in cluster $n$ is denoted by $\Gamma^{U}_{n}=\max_k \frac{Q\hat{Q}}{\bar{U}^{*n}_{k}}$,  $\forall k\in \mathcal{C}_n$. Similarly, let $\bar{R}^{*n}_{k}$ the maximum average DL rate of MU $k\in\mathcal{C}_n$. The DL latency of gradient aggregation in cluster $n$ is denoted by $\Gamma^{D}_{n}=\max_k \frac{Q\hat{Q}}{\bar{R}^{*n}_{k}}$,  $\forall k\in \mathcal{C}_n$.

After $H$ iterations, SBSs send the model to the MBS for the purpose of averaging the clusters model. Let $U^{SBS}$, $R^{SBS}$ be the UL, DL rate of SBSs to the MBS, respectively. The UL, DL latency at each period of $H$ iterations become $\Theta^{U}=\frac{Q\hat{Q}}{U^{SBS}}$ and $\Theta^{D}=\frac{Q\hat{Q}}{R^{SBS}}$, respectively. There is also the latency of transmitting the average model by SBSs to their associated MUs. The average latency associated with a period of hierarchical distributed SGD becomes.

\begin{align}
    \Gamma^{period} = \max_{n\in \mathcal{N} }\left(\sum^{H}_{i=1} \Gamma^{U}_{n}(i)+\Gamma^{D}_{n}(i)\right)+\Theta^{U}+\Theta^{D} + \max_n(\Gamma^D_n),
\end{align}
where $T^{U}_{n}(i)$ and $T^{D}_{n}(i)$ is the latency of $i$-th iteration of UL and DL aggregation in cluster $n$, respectively. The average per iteration latency of HFL becomes $\Gamma^{HFL} = \frac{\Gamma^{period}}{H}$.

% notation table

\section{Sparse Communications}
\begin{algorithm}
\caption{Sparse Federated SGD algorithm \cite{SGD_sparse1}}\label{alg:noc}
\begin{algorithmic}[1]
\State Initialize $\mathcal{D}_k$, learning rate $\eta$, $\mathbf{w}_k=\mathbf{w}_0$, sparse factor $\phi^{ul}_{MU}$
\State Each MU $k$ does the following
\For{$t=1,2,\ldots$}
\State Randomly select a mini-batch $\mathcal{I}_k\subseteq\mathcal{D}_k$
\State Calculate $\mathbf{g}_{k,t}=\frac{1}{|\mathcal{I}^{k}|}\sum_{i\in\mathcal{I}^{k}}\nabla f_i(\mathbf{w}_t)$
\State $\mathbf{u}_{k,t} = \sigma \mathbf{u}_{k,t-1} + \mathbf{g}_{k,t}$
\State $\mathbf{v}_{k,t} = \mathbf{v}_{k,t-1} + \mathbf{u}_{k,t}$
\State $g_{th} \leftarrow \phi^{ul}_{MU}$ of $|\mathbf{v}_{k,t}|$
\State $mask \leftarrow |\mathbf{v}_{k,t}|\geq g_{th}$
\State $\hat{\mathbf{g}}_{k,t} = \mathbf{v}_{k,t} \odot mask$
\State $\mathbf{u}_{k,t} = \mathbf{u}_{k,t} \odot \neg mask$
\State $\mathbf{v}_{k,t} = \mathbf{v}_{k,t} \odot \neg mask$
\State send $\hat{\mathbf{g}}_{k,t}$ to MBS
\State receive $\mathbf{g}_{t} = \sum^{K}_{k=1} \hat{\mathbf{g}}_{k,t}$
\State $\mathbf{w}_{t+1} = \mathbf{w}_t - \eta_t \mathbf{g}_{t}$
\EndFor
\end{algorithmic}
\end{algorithm}

The trend of going deeper in the \emph{depth} of the neural networks for increasing the accuracy have resulted in NNs with tens of millions of parameters. The amount of data required to be communicated is challenging even for cable connections let alone the wireless links. On top of performing periodical parameter averaging, sparse communication can be used to significantly improve the latency. In sparsification the fact that the gradient vector is sparse is used to transmit only a fraction (i.e., $1-\phi$ ) of parameters and considerably reduce latency.

To make sure that all gradients are transmitted eventually, a separate parameter vector, $\mathbf{v}$ is used to accumulate the error associated with the gradients that are not transmitted. The gradients that are not transmitted will grow in time, as recorded by $\mathbf{v}$, and eventually will be transmitted. More specifically the error buffer 
is calculated as
\begin{align}
    \mathbf{v}_{k,t} = \mathbf{v}_{k-1,t} + \mathbf{g}_{k,t}
\end{align}
Now each MU $k$, instead of transmitting $\mathbf{g}_{k,t}$ transmits $sparse(\mathbf{v}_{k,t})$ to be aggregated at the MBS (or SBSs in the clusters).

Note that the vanilla SGD used an Algorithm \ref{alg:dist} and \ref{alg:dist-hier} is the simplest form of a optimizer and its performance is quite poor in large scale optimization problems. An efficient way of accelerating the performance of vanilla SGD is to apply momentum method. In momentum method the parameters are updated as following:
\begin{align}\label{eq:momsgd}
    \mathbf{u}_t = \sigma \mathbf{u}_{t-1} + \mathbf{g}_t \nonumber\\
    \mathbf{w}_{t+1} = \mathbf{w}_{t} - \eta_t \mathbf{u}_t,
\end{align}
where $\sigma$ is the momentum and $\mathbf{g}$ is the aggregated gradient.

Directly applying momentum to the sparsed gradients will result in a poor performance and momentum correction is required. Here, we directly employ the method preoposed in \cite{SGD_sparse1}

\begin{align}
    &\mathbf{u}_{k,t} = \sigma \mathbf{u}_{k,t-1} + \mathbf{g}_{k,t}\\
    &\mathbf{v}_{k,t} = \mathbf{v}_{k-1,t} + \mathbf{u}_{k,t}\\
    &\mathbf{w}_{t+1} = \mathbf{w}_{t} - \frac{\eta}{K}\sum^{K}_{k=1}sparse(\mathbf{v}_{k,t})
\end{align}
Sparsification delays transmitting gradients that are too small. When they are finally transmitted they become stale and slow down the convergence. To combat the staleness \cite{SGD_sparse1}, we apply the inverted sparsification to both accumulated gradients and momentum factor as follows:
\begin{align}
    &mask \leftarrow |\mathbf{v}_{k,t}|\geq g_{th}\\
    &\mathbf{u}_{k,t} = \mathbf{u}_{k,t} \odot \neg mask\\
    &\mathbf{v}_{k,t} = \mathbf{v}_{k,t} \odot \neg mask
\end{align}
The mask simply prevents the stale momentums to be applied. The detailed algorithms for sparse federated SGD is represented in Algorithm \ref{alg:noc}.

\begin{table}[]
\caption{ Parameters}
\begin{tabular}{|p{2cm}|p{6cm}|}
\hline
 Parameter & Definition\\ \hline
 $\phi^{ul}_{MU}$, $\phi^{dl}_{SBS}$, $\phi^{ul}_{SBS}$, $\phi^{dl}_{MBS}$     &   Sparsification parameters for  uplink from MU to SBS, downlink from SBS to MU, uplink from SBS to MBS and downlink from MBS to SBS.  \\ \hline
 $\mathbf{e}$, $\mathbf{e}_{n}$, $\mathbf{\epsilon}_{n}$ &  Model errors due to sparsification before downlink from MBS to SBSs, downlink from $SBS_{n}$ to MU and uplink from $SBS_{n}$ to MBS respectively. \\ \hline
  $\mathbf{w}_{k}$, $\mathbf{W}_{n}$ and $\mathbf{W}$ &   Parameter model at $k$th MU, $n$th SBS and MBS respectively.\\ \hline
$\Delta_{\mathbf{W}}$, $\Delta_{\mathbf{W}_n}$ and $\delta\mathbf{w}_{n}$ &  Model difference send to SBSs from MBS, to MBS from $SBS_{n}$ and to MUs from $SBS_{N}$ respectively.\\ \hline
\end{tabular}
\end{table}

\subsection{Sparse Communication and Error Accumulation}
Our proposed HFL framework consists of 4 communication steps: uplink from MU to SBS, downlink from SBS to MU, uplink from SBS to MBS and downlink from MBS to SBS. For each communication step, we employ different sparsification parameters, $\phi^{ul}_{MU}$, $\phi^{dl}_{SBS}$, $\phi^{ul}_{SBS}$ and $\phi^{dl}_{MBS}$ respectively, to speed up the communication. We introduce the function $\Omega(\mathbf{V},\phi):\mathcal{R}^d \rightarrow \mathcal{R}^d$, which maps a $d$ dimensional vector to its sparse form where only $1-\phi$ portion of the indicies have non-zero values.\\
\indent The sparsification procedure in each step leads to an error in the parameter model and thus slows down the convergence. To overcome this issue we employ the discounted error accumulation technique, similar to \cite{SGD_q4,fedlearn7}, which uses the discounted version of the error for the next model update. Before the details of the error accumulation strategy, we want to introduce the following parameters $\mathbf{w}_{n}$ and $\mathbf{W}$ which denotes the parameter model at $n$th SBS and MBS respectively. We note that to employ sparsification for model averaging it is more convenient to transmit the model difference rather than the model. To this end, we introduce the reference models $\tilde{\mathbf{W}}_{n}$ and $\tilde{\mathbf{W}}$ for SBSs and MBS respectively so that each SBS sends the model difference based on $\tilde{\mathbf{W}}$ to the the MBS and based on $\tilde{\mathbf{W}}_{n}$ to the corresponding MUs. In the proposed HFL framework, we use $\mathbf{e}$, $\mathbf{e}_{n}$, and $\mathbf{\epsilon}_{n}$ to keep the model errors due to sparsification before downlink from MBS to SBSs, downlink from $SBS_{n}$ to MU and uplink from $SBS_{n}$ to MBS respectively. The overall procedure for the HFL framework is illustrated  in Algorithm \ref{alg:c}, where error accumulation strategy is employed at lines 21, 28, 34 and parameters $\beta_{m}$ and $\beta_{s}$ are the discount factors for the error accumulation.

\begin{algorithm}
\caption{Hierarchical Federated Leraning}\label{alg:c}
\begin{algorithmic}[1]
\State Initialize $\mathbf{W}$ and $\tilde{\mathbf{W}}$
\For{$n=1,2,\ldots N$}
\State Initialize $\mathbf{W}_{n}$ and $\tilde{\mathbf{W}_{n}}$
\EndFor
\For{$k=1,2,\ldots K$}
\State Initialize $\mathbf{w}_{k}$
\EndFor
\For{$t=1,\ldots T-1$}
\State \textbf{Computation and Uplink}:
\For{$k=1,2,\ldots K$}
\State Randomly select a mini-batch $\mathcal{I}_k\subseteq\mathcal{D}_k$
\State Calculate $\mathbf{g}_{k,t}=\frac{1}{|\mathcal{I}^{k}|}\sum_{i\in\mathcal{I}^{k}}\nabla f_i(\mathbf{w}_t)$
\State $\mathbf{u}_{k,t} = \sigma \mathbf{u}_{k,t-1} + \mathbf{g}_{k,t}$
\State $\mathbf{v}_{k,t} = \mathbf{v}_{k,t-1} + \mathbf{u}_{k,t}$
\State $g_{th} \leftarrow \phi^{ul}_{MU}$ of $|\mathbf{v}_{k,t}|$
\State $mask \leftarrow |\mathbf{v}_{k,t}|\geq g_{th}$
\State $\hat{\mathbf{g}}_{k,t} = \mathbf{v}_{k,t} \odot mask$
\State $\mathbf{u}_{k,t} = \mathbf{u}_{k,t} \odot \neg mask$
\State $\mathbf{v}_{k,t} = \mathbf{v}_{k,t} \odot \neg mask$
\State send $\hat{\mathbf{g}}_{k,t}$ to associated $SBS$
\EndFor
\State \textbf{Model Average}:
\For{$n=1,2,\ldots N$}
\State Update $\mathbf{W}_n(t+1) =  \tilde{\mathbf{W}}_n(t) -\eta \hat{\mathbf{g}}_{n}  + \beta_{s} \mathbf{\epsilon}_{n}(t)$
\EndFor
\If{$t$ is divisible by $H$}
\For{$n=1,2,\ldots N$}
\State $\Delta_{\mathbf{W}_n}(t+1) = \mathbf{W}_n(t+1)- \tilde{\mathbf{W}}(h)$
\State send $\Omega(\Delta_{\mathbf{W}_n}(t+1),\phi^{ul}_{SBS})$ to MBS
\State $\mathbf{e}_n(t+1)=$ 
\State $\Delta_{\mathbf{W}_n}(t+1) -\Omega(\Delta_{\mathbf{W}_n}(t+1),\phi^{ul}_{SBS})$
\EndFor
\State $\Delta_{\mathbf{W}} = \sum_{n} \Omega(\Delta_{\mathbf{W}_n}(t+1),\phi^{ul}_{SBS})+\beta_{m}e$
\State MBS transmit $\Omega(\Delta_{\mathbf{w}},\phi^{dl}_{MBS})$ to all SBSs
\State $\mathbf{e}= \Delta_\mathbf{W}-\Omega(\Delta_{\mathbf{W}},\phi^{dl}_{MBS})$
\State $\tilde{\mathbf{W}}(h+1) = \tilde{\mathbf{W}}(h) + \Omega(\Delta_{\mathbf{W}},\phi^{dl}_{MBS})$
\For{$n=1,2,\ldots N$}
\State $\mathbf{W}_n(t+1) =$
\State $\tilde{\mathbf{W}}(h) + \Omega(\Delta_\mathbf{W},\phi^{dl}_{MBS})+\mathbf{e}_n(t+1)/N$
\EndFor
\EndIf
\For{$n=1,2,\ldots N$}
\State $\delta_{\mathbf{W}_n}(t+1)=\mathbf{W}_n(t+1)-\tilde{W}_n(t)$
\State $SBS_{n}$ sends $\Omega(\delta_{\mathbf{W}_n}(t+1),\phi^{dl}_{SBS})$ to MUs
\State $\tilde{\mathbf{W}}_n(t+1)=\tilde{\mathbf{W}}_n(t)+ \Omega(\delta_{\mathbf{W}_n}(t+1),\phi^{dl}_{SBS})$
\State $\mathbf{\epsilon}_{n}(t+1)= \delta_{\mathbf{W}_n}(t+1)-\Omega(\delta_{\mathbf{W}_n}(t+1),\phi^{dl}_{SBS})$
\EndFor
\State \textbf{Update}:
\For{$n=1,2,\ldots N$}
\For{$k\in \mathcal{S}_{n}$}
\State $\mathbf{w}_{k}(t+1)=\tilde{\mathbf{W}}_{n}(t+1)$
\EndFor
\EndFor

\EndFor
\end{algorithmic}
\end{algorithm}

\section{Numerical Results}

\begin{table}[]
\caption{Simulation Parameters}
\begin{tabular}{|l|l|}
\hline
 Number of sub-carriers & $M=600$\\ \hline
 Sub-carrier spacing     &   $30KHz$ \\ \hline
 Noise power & $-150$dB \\ \hline
 MBS Tx power & $20$W \\ \hline
 SBS Tx Power & $6.3$W \\ \hline
 MU Tx power & $0.2$W \\ \hline
 Path-loss exponent & 2.8 \\ \hline
 BER & $10^{-3}$ \\ \hline
\end{tabular}
\end{table}

\subsection{Network topology}
We consider a circular area with radius  $750$ meters where users are generated uniformly randomly inside it. We consider hexagonal clusters where the diameter of circle inscribed is $500$ meters. The SBSs are exactly resided in the center of the hexagons. To mitigate the interference between the clusters, we use a simple reuse pattern \cite{reuse} as shown in Figure \ref{fig:cluster}. We assume that the fronthaul link is $100$ times faster than the UL, DL between MUs and SBSs \footnote{ For a $8\times8$ MIMO the fornthaul rate estimate is $8$ Gbps with 3GPP split option 2 and $67$ Gbps with 3GPP split option 7 \cite{5gfronthaul}.}.  Total number of clusters are $7$.

\begin{figure}
\centering
     \includegraphics[scale=0.45]{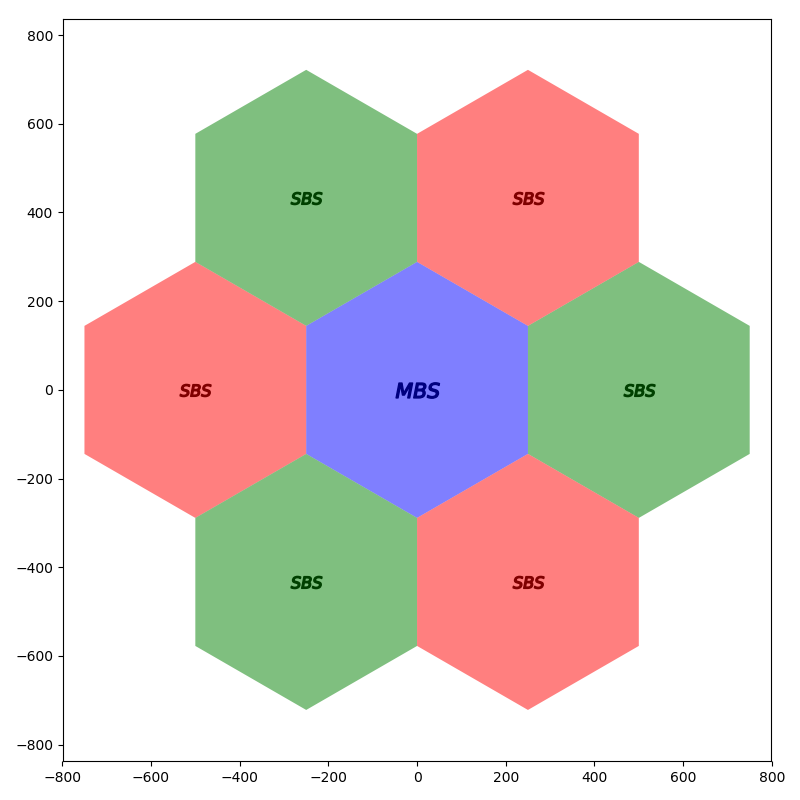}
				\caption{Clustering layout. Frequency reuse pattern is one. Each color illustrates distinct set of subcarriers.}
		\label{fig:cluster}
\end{figure}
There are $300$ sub carriers with sub carrier spacing of $30$ KHz. The maximum transmit powers of MBS and SSBs are $20$ and $6.3$W, respectively and maximum transmit power of MUs is $0.2$W \cite{earth}.

\subsection{Implementation guideline}
In our numerical analysis, we consider the  CIFAR-10 dataset for image classification problem with 10 different image classes  \cite{dataset2} and train the  ResNet18  architecture \cite{training2}. For further details regarding the trained NN structures please refer to \cite{nn}.\\
\indent For the simulations, we also utilize some large batch training tricks such as scaling the learning rate $\eta$ and employing a warm-up phase \cite{training1}. In all simulations, data sets are divided among the MUs without any shuffling and through the iterations MUs train the same subset of the dataset as in the FL framework and we set the batch size for training to $\beta=64$. In general, batch size $K\times \beta=128$ is accepted as the baseline batch size with the corresponding learning rate $\eta=0.1$ and the initial learning rate is adjusted according to the cumulative batch size $K\times\beta$ accordingly \cite{training1,training2}. Hence, we set the initial learning to $0.25$, also we consider the first $5$ epoch as the gradual warm-up phase where training starts with a small learning rate and then increased linearly at each iteration until it reaches the value of the given initial learning rate. For the network training, we follow the guideline in \cite{sgd_local4},  we train the network for 300 epochs,  and at the end of 150th epoch  we drop the initial learning rate by factor of 10, and similarly end of the 225th epoch we drop the learning rate by factor of 10 again. Further, for the weight decay factor\footnote{We dont apply weight decay to batch normalization parameters} and the momentum term we use $w=0.0001$ and $\sigma=0.9$ respectively in all simulations. Finally, for the discounted error accumulation we set $\beta_{m}=0.2$ and $\beta_{s}=0.5$.

\subsection{Results}
We first study the amount of speed up in latency achieved by HFL versus FL. We measure the speed by comparing the per iteration latency of HFL i.e., $\Gamma^{HFL}$ and FL, $T^{FL}$. More specifically, speed up $=\frac{T^{FL}}{\Gamma^{HFL}}$. By varying the number of MUs in each cluster, and for different periods of $H=2,4,6$, we plot the speed up achieved by HFL, when sparsity parameters $\phi^{ul}_{MU}=0.99$, $\phi^{dl}_{SBS}=0.9$, $\phi^{ul}_{SBS}=0.9$, $\phi^{dl}_{MBS}=0.9$  are used, in Figure \ref{fig:latency_comp}. We observe that HFL achieves good latency speed up with respect to FL and it improves when the period increases. 

\begin{figure}
\centering
     \includegraphics[scale=0.5]{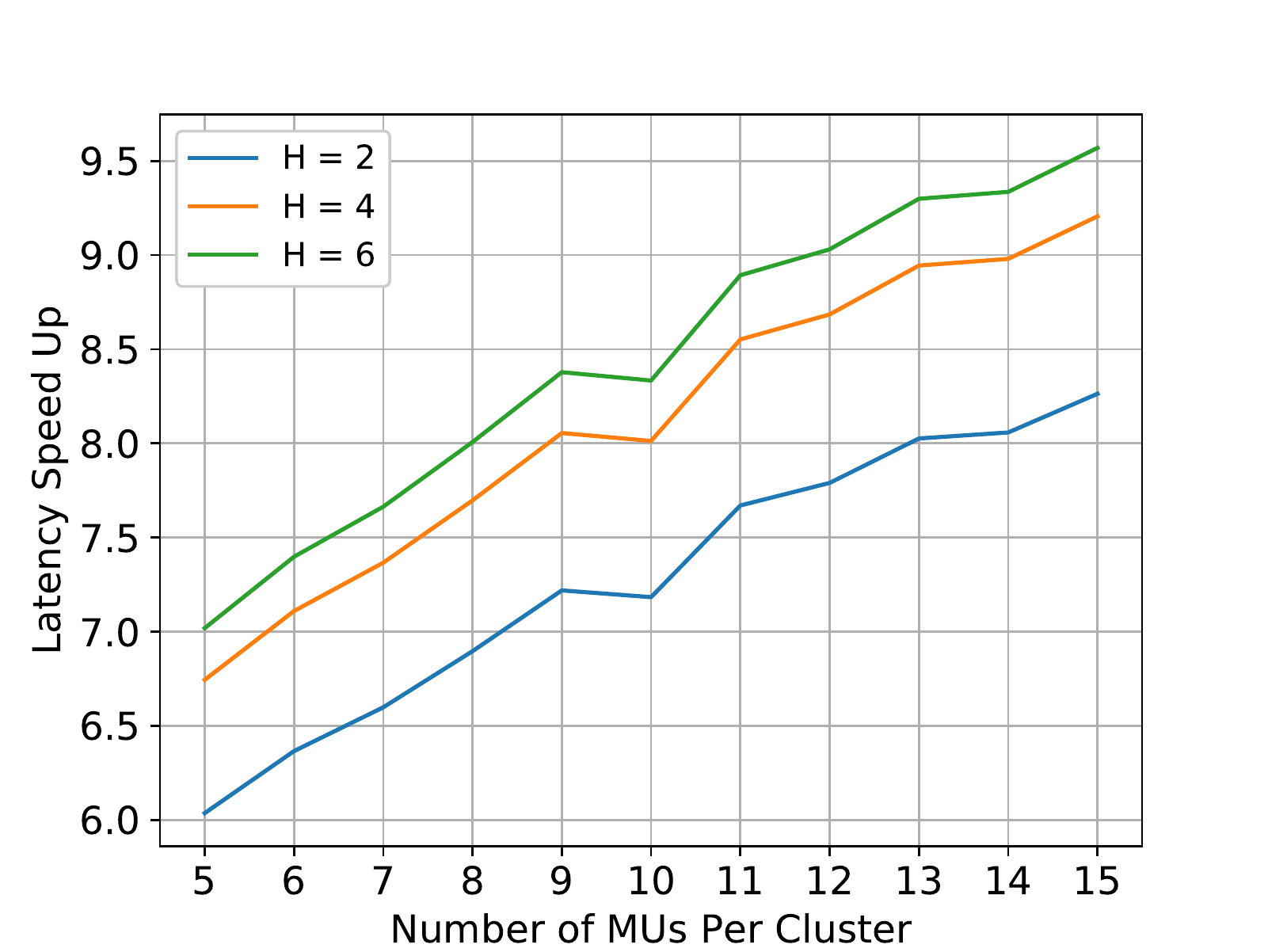}
				\caption{Latency Speed Up HFL versus FL.}
		\label{fig:latency_comp}
\end{figure}

Clustering reduces the communication distance and as a result improves the SNR. The amount of improvement depends on the amount of reduction in path-loss. In Figure \ref{fig:latency_comp_pathloss}, we illustrate the amount of latency speed up due to clustering as a function of the path-loss exponent, $\alpha$. When the path-loss exponent is increased, the SNR in centralized scheme is punished more severely than clustering scheme due to longer communication paths. Thus, the latency speed up improves when the path-loss is more severe.
\begin{figure}
\centering
     \includegraphics[scale=0.5]{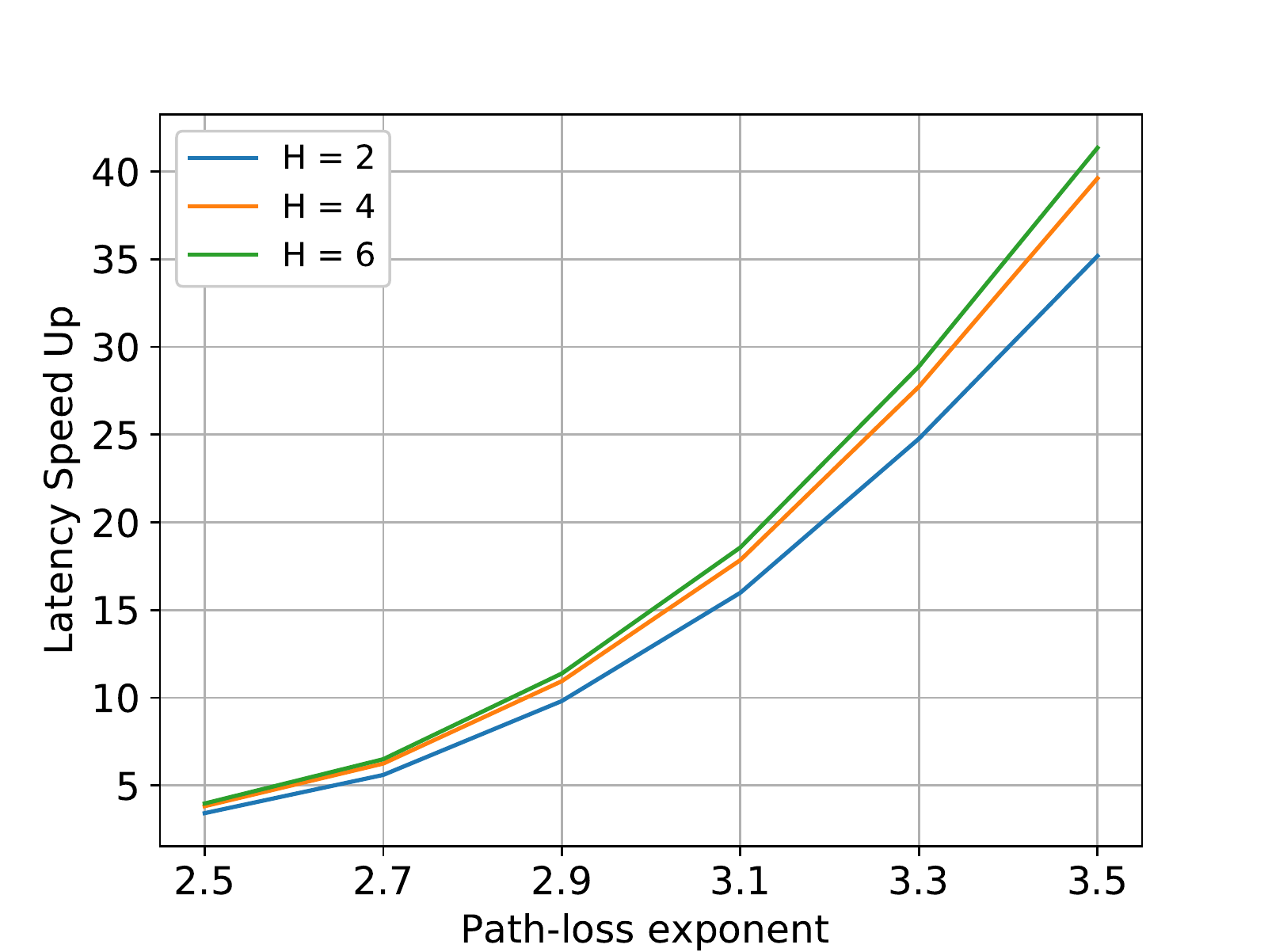}
				\caption{Latency speed Up HFL versus FL as a function of path-loss exponent $\alpha$.}
		\label{fig:latency_comp_pathloss}
\end{figure}

To see the importance of sparsification, we compare the HFL and FL with sparse HFL and sparse FL in Figure \ref{fig:sparse_comp_c} and \ref{fig:sparse_comp_noc}, respectively. For both FL and HFL the sparsification provides a significant improvment. However, the latency improvement in HFL is more robust with respect to increasing the number of MUs. This is due to the fact that MBS is serving more number of MUs than a single MBS, and hence the scarcity of resources in macro cell has more impact to the latency than small cells.

The Top-1 accuracy achieved by FL and HFL algorithms are illustrated in Figure \ref{fig:resnetacc} for CIFAR-10 data set trained with ResNet 18. We observe that latency speed up delivered by HFL over FL schemes does not compromise the accuracy of the ML model. In fact, a closer look at the accuracy (average over $5$ runs) presented in Table \ref{tab:acc} show that HFL is able to achieve a better accuracy than FL in all situations. The $\emph{\text{mean}}\pm \emph{\text{standard error mean}}$ results for the last epoch is reprted in Table \ref{tab:acc}, where the \emph{Baseline} result is obtained by training a single MU on the whole training set. We observe a small degradation in the accuracy for our introduced hierarchical distributed learning strategy. We conjecture that this degradation is mainly due to the use of momentum SGD at each MU instead of a global momentum and due to sparsification. We also observe that FL, based on \cite{SGD_sparse1}, performs poorly compared to HFL. We believe that this poor performance is mainly due to the downlink sparsification that we consider.\\ 
\indent We believe that using an additional global momentum term \cite{sgd_local4} or utilizing momentum averaging strategy \cite{sgd_local3} accuracy of the HFL can be improved further.

\begin{figure}[ht]
\begin{subfigure}{.55\textwidth}
  \centering
  % include first image
  \includegraphics[width=.8\linewidth]{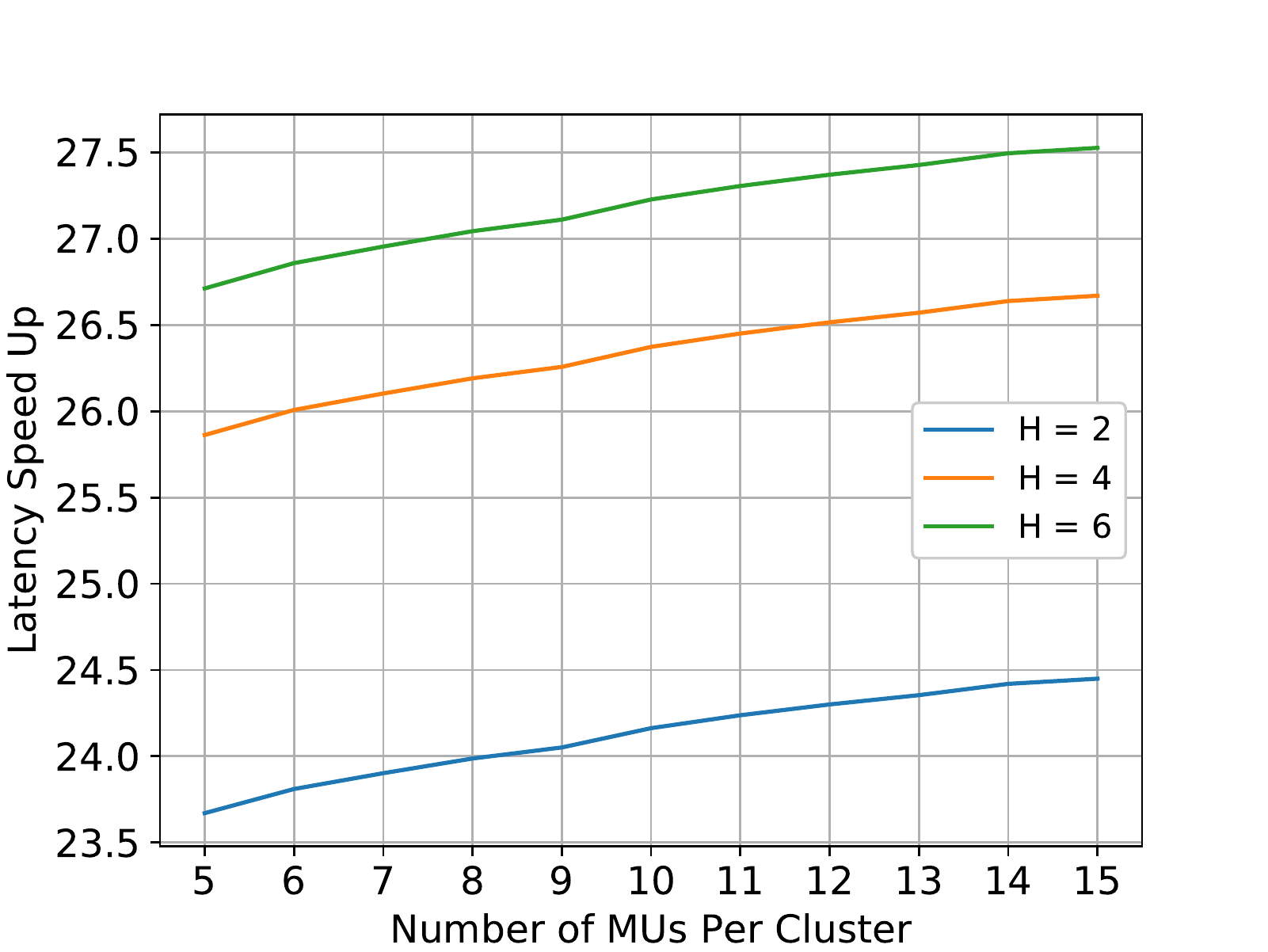}  
  \caption{HFL}
  \label{fig:sparse_comp_c}
\end{subfigure}
\begin{subfigure}{.55\textwidth}
  \centering
  % include second image
  \includegraphics[width=.8\linewidth]{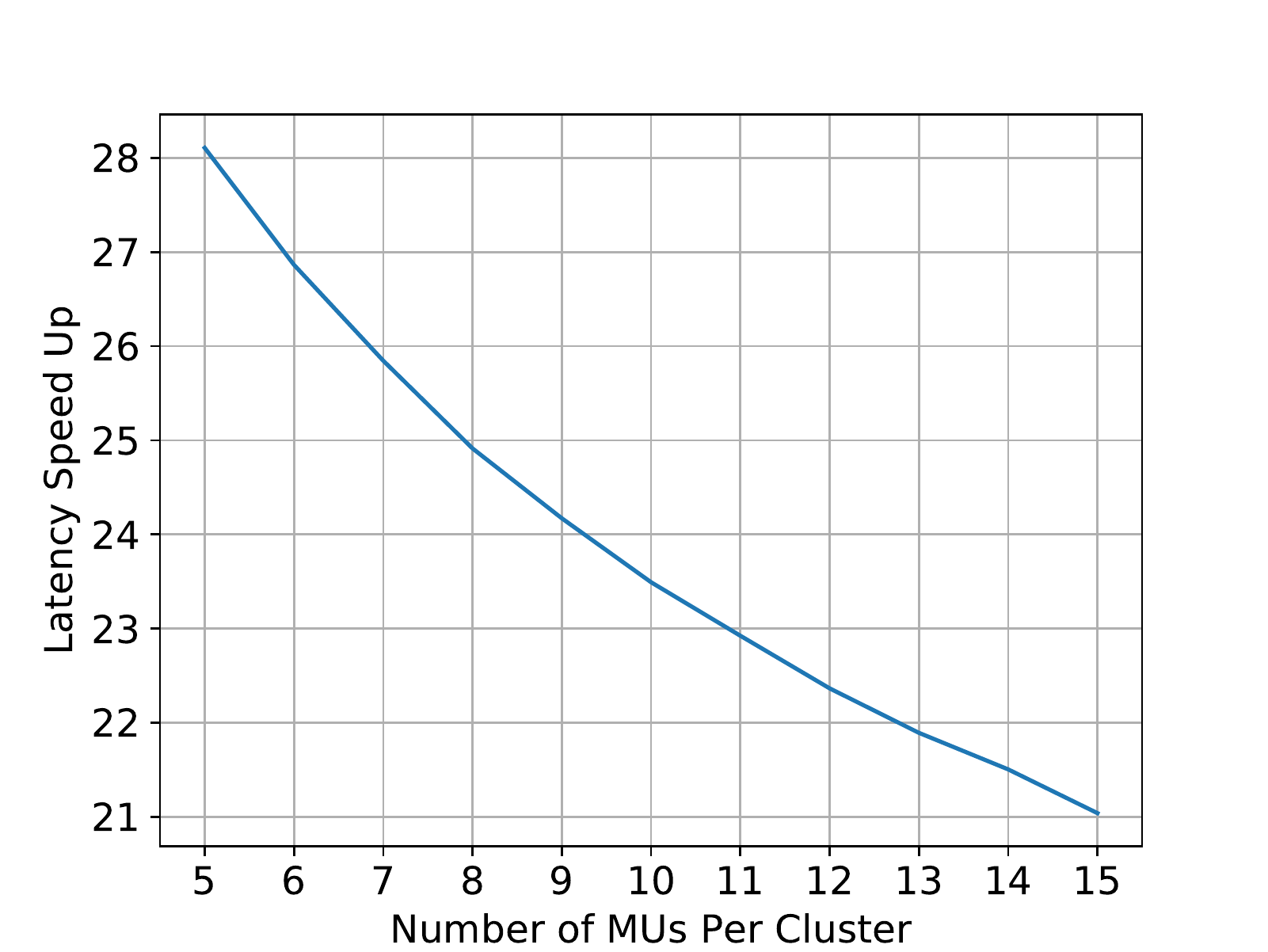}  
  \caption{FL}
  \label{fig:sparse_comp_noc}
\end{subfigure}
\caption{Latency speed up due to sparsification}
\label{fig:sparse_comp}
\end{figure}

% \begin{figure}
% \centering
%      \includegraphics[scale=0.5]{fmnistacc.eps}
% 				\caption{Top-1 accuracy }
% 		\label{fig:fmnistacc}
% \end{figure}

% \begin{figure}
% \centering
%      \includegraphics[scale=0.5]{cifaracc.eps}
% 				\caption{Top-1 accuracy }
% 		\label{fig:cifaracc}
% \end{figure}

\begin{figure}
\centering
     \includegraphics[scale=0.5]{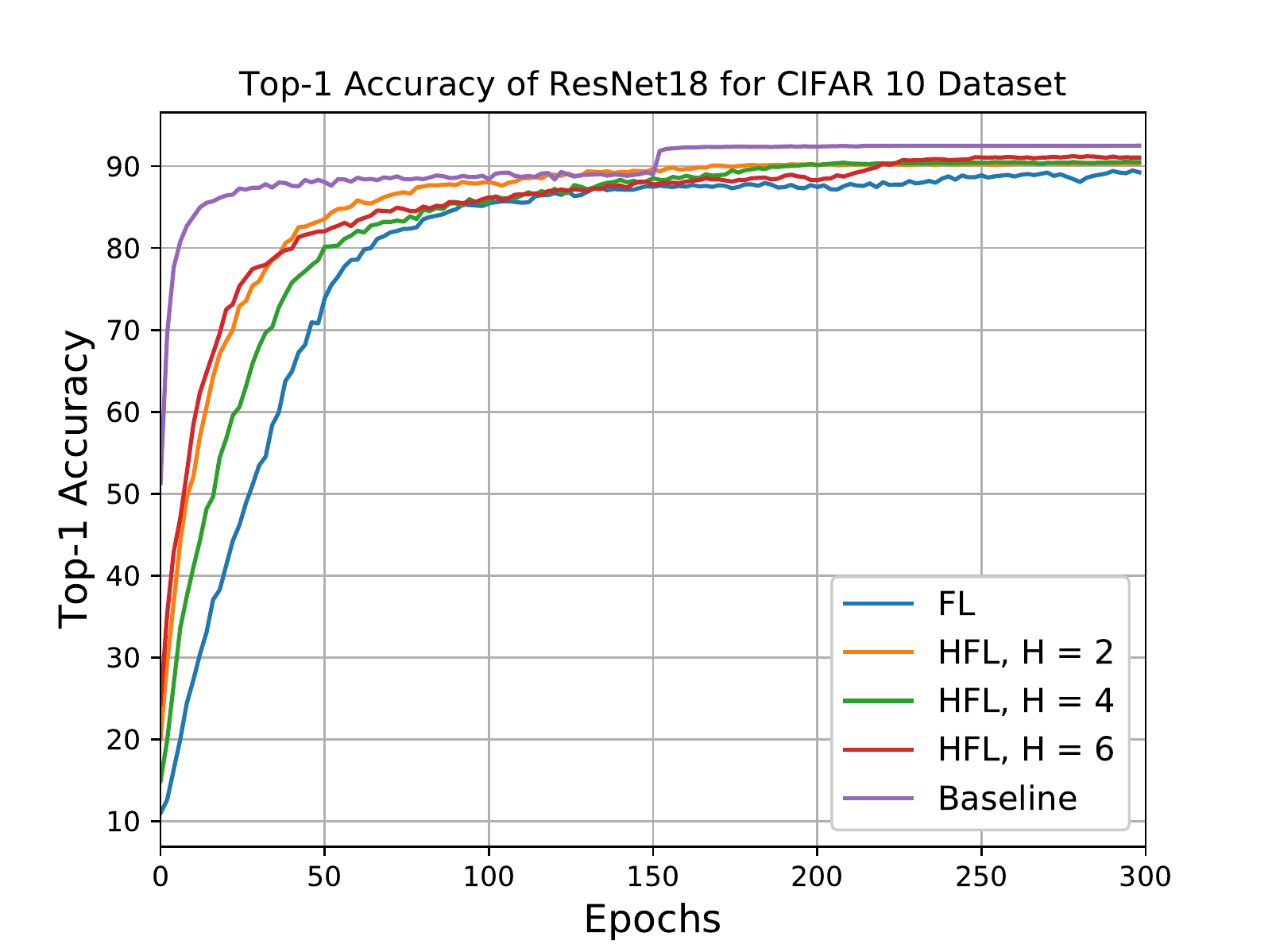}
				\caption{Top-1 accuracy }
		\label{fig:resnetacc}
\end{figure}

 \begin{table*}[]
 \caption{Top 1 accuracy results obtained by different strategies}
 \label{tab:acc}
\begin{center}
    \begin{tabular}{ |  c | c |c |c|}
    \hline
 		  CIFAR-10 \cite{dataset2} / ResNet18 \cite{training2}  & Baseline  & $92.48\pm 0.13$& -\\ \hline 
           CIFAR-10 \cite{dataset2} / ResNet18 \cite{training2}  & FL & $89.23\pm 0.42$& 28 MUs\\ \hline 
          CIFAR-10 \cite{dataset2} / ResNet18 \cite{training2}  & HFL, $H=2$ & $90.27\pm 0.11$& $7$ clusters $\times$ $4$ MUs\\ \hline 
           CIFAR-10 \cite{dataset2} / ResNet18 \cite{training2}  & HFL, $H=4$ & $90.474\pm 0.20$& $7$ clusters $\times$ $4$ MUs\\ \hline 
           CIFAR-10 \cite{dataset2} / ResNet18 \cite{training2}  & HFL, $H=6$ & $91.03\pm 0.19$& $7$ clusters $\times$ $4$ MUs\\ \hline 
     \end{tabular}		
 \end{center}
\end{table*}
\subsection{Discussions}
 As mentioned above, a momentum correction strategy for the clusters can help to improve the accuracy as weel as increasing the convergence speed We will consider this as a future work.\\
\indent  In addition, one of the recent research area regarding the federated learning framework is the non-IID data distribution among the MUs \cite{fedlearn5,fedlearn7,fedlearn8,fedlearn9} hence, we are planning to extend our study to the non-IID distribution scenario as well. Finally, another research direction that we are planning to investigate is the optimal batchsize for MUs. In \cite{training3}, it has been shown that learning speed increases with the batchsize until certain point, hence we believe that a good strategy for federated learning is to adjust the batchsize according to the number users. We are planing to extend our analysis in this direction and do an extended analysis on training time including the computation time of the MUs.

\bibliographystyle{IEEEtran}
\bibliography{IEEEabrv,ref}

\begin{thebibliography}{10}
\providecommand{\url}[1]{#1}
\csname url@samestyle\endcsname
\providecommand{\newblock}{\relax}
\providecommand{\bibinfo}[2]{#2}
\providecommand{\BIBentrySTDinterwordspacing}{\spaceskip=0pt\relax}
\providecommand{\BIBentryALTinterwordstretchfactor}{4}
\providecommand{\BIBentryALTinterwordspacing}{\spaceskip=\fontdimen2\font plus
\BIBentryALTinterwordstretchfactor\fontdimen3\font minus
  \fontdimen4\font\relax}
\providecommand{\BIBforeignlanguage}[2]{{%
\expandafter\ifx\csname l@#1\endcsname\relax
\typeout{** WARNING: IEEEtran.bst: No hyphenation pattern has been}%
\typeout{** loaded for the language `#1'. Using the pattern for}%
\typeout{** the default language instead.}%
\else
\language=\csname l@#1\endcsname
\fi
#2}}
\providecommand{\BIBdecl}{\relax}
\BIBdecl

\bibitem{ML_overair2}
M.~M. Amiri and D.~G{\"{u}}nd{\"{u}}z, ``Machine learning at the wireless edge:
  Distributed stochastic gradient descent over-the-air,'' \emph{CoRR}, vol.
  abs/1901.00844, 2019.

\bibitem{ML_overair3}
K.~Yang, T.~Jiang, Y.~Shi, and Z.~Ding, ``Federated learning via over-the-air
  computation,'' \emph{CoRR}, vol. abs/1812.11750, 2018.

\bibitem{ML_overair4}
G.~Zhu, Y.~Wang, and K.~Huang, ``Low-latency broadband analog aggregation for
  federated edge learning,'' \emph{CoRR}, vol. abs/1812.11494, 2018.

\bibitem{int_edge1}
\BIBentryALTinterwordspacing
J.~Park, S.~Samarakoon, M.~Bennis, and M.~Debbah, ``Wireless network
  intelligence at the edge,'' \emph{CoRR}, vol. abs/1812.02858, 2018. [Online].
  Available: \url{http://arxiv.org/abs/1812.02858}
\BIBentrySTDinterwordspacing

\bibitem{fedlearn1}
\BIBentryALTinterwordspacing
B.~McMahan, E.~Moore, D.~Ramage, S.~Hampson, and B.~A. y~Arcas,
  ``{Communication-Efficient Learning of Deep Networks from Decentralized
  Data},'' in \emph{Proceedings of the 20th International Conference on
  Artificial Intelligence and Statistics}, ser. Proceedings of Machine Learning
  Research, A.~Singh and J.~Zhu, Eds., vol.~54.\hskip 1em plus 0.5em minus
  0.4em\relax Fort Lauderdale, FL, USA: PMLR, 20--22 Apr 2017, pp. 1273--1282.
  [Online]. Available: \url{http://proceedings.mlr.press/v54/mcmahan17a.html}
\BIBentrySTDinterwordspacing

\bibitem{fedlearn2}
\BIBentryALTinterwordspacing
K.~Bonawitz, H.~Eichner, W.~Grieskamp, D.~Huba, A.~Ingerman, V.~Ivanov,
  C.~Kiddon, J.~Konecn{\'{y}}, S.~Mazzocchi, H.~B. McMahan, T.~V. Overveldt,
  D.~Petrou, D.~Ramage, and J.~Roselander, ``Towards federated learning at
  scale: System design,'' \emph{CoRR}, vol. abs/1902.01046, 2019. [Online].
  Available: \url{http://arxiv.org/abs/1902.01046}
\BIBentrySTDinterwordspacing

\bibitem{ML_overair1}
L.~{Chen}, N.~{Zhao}, Y.~{Chen}, F.~R. {Yu}, and G.~{Wei}, ``Communicating or
  computing over the mac: Function-centric wireless networks,'' \emph{IEEE
  Transactions on Communications}, pp. 1--1, 2019.

\bibitem{ML_overair5}
\BIBentryALTinterwordspacing
G.~Zhu, D.~Liu, Y.~Du, C.~You, J.~Zhang, and K.~Huang, ``Towards an intelligent
  edge: Wireless communication meets machine learning,'' \emph{CoRR}, vol.
  abs/1809.00343, 2018. [Online]. Available:
  \url{http://arxiv.org/abs/1809.00343}
\BIBentrySTDinterwordspacing

\bibitem{ML_overair6}
\BIBentryALTinterwordspacing
D.~G{\"{u}}nd{\"{u}}z, P.~de~Kerret, N.~D. Sidiropoulos, D.~Gesbert, C.~Murthy,
  and M.~van~der Schaar, ``Machine learning in the air,'' \emph{CoRR}, vol.
  abs/1904.12385, 2019. [Online]. Available:
  \url{http://arxiv.org/abs/1904.12385}
\BIBentrySTDinterwordspacing

\bibitem{ML_overair7}
J.-H. Ahn, O.~Simeone, and J.~Kang, ``Wireless federated distillation for
  distributed edge learning with heterogeneous data,'' \emph{CoRR}, vol.
  abs/1907.02745, 2019.

\bibitem{ML_overair8}
\BIBentryALTinterwordspacing
M.~M. Amiri, T.~M. Duman, and D.~G{\"{u}}nd{\"{u}}z, ``Collaborative machine
  learning at the wireless edge with blind transmitters,'' \emph{CoRR}, vol.
  abs/1907.03909, 2019. [Online]. Available:
  \url{http://arxiv.org/abs/1907.03909}
\BIBentrySTDinterwordspacing

\bibitem{FL_wireless1}
J.~Ren, G.~Yu, and G.~Ding, ``Accelerating {DNN} training in wireless federated
  edge learning system,'' \emph{CoRR}, vol. abs/1905.09712, 2019.

\bibitem{FL_wireless2}
N.~H. {Tran}, W.~{Bao}, A.~{Zomaya}, N.~{Minh N.H.}, and C.~S. {Hong},
  ``Federated learning over wireless networks: Optimization model design and
  analysis,'' in \emph{IEEE INFOCOM 2019 - IEEE Conference on Computer
  Communications}, April 2019, pp. 1387--1395.

\bibitem{sgd_local4}
T.~Lin, S.~U. Stich, and M.~Jaggi, ``Don't use large mini-batches, use local
  {SGD},'' \emph{CoRR}, vol. abs/1808.07217, 2018.

\bibitem{sgd_local5}
F.~Zhou and G.~Cong, ``A distributed hierarchical {SGD} algorithm with sparse
  global reduction,'' \emph{CoRR}, vol. abs/1903.05133, 2019.

\bibitem{fedlearn8}
\BIBentryALTinterwordspacing
L.~Liu, J.~Zhang, S.~H. Song, and K.~B. Letaief, ``Edge-assisted hierarchical
  federated learning with non-iid data,'' \emph{CoRR}, vol. abs/1905.06641,
  2019. [Online]. Available: \url{http://arxiv.org/abs/1905.06641}
\BIBentrySTDinterwordspacing

\bibitem{goldsmith}
A.~J. {Goldsmith} and S.-G. Chua, ``Variable-rate variable-power mqam for
  fading channels,'' \emph{IEEE Transactions on Communications}, vol.~45,
  no.~10, pp. 1218--1230, Oct 1997.

\bibitem{fedlearn6}
S.~{Wang}, T.~{Tuor}, T.~{Salonidis}, K.~K. {Leung}, C.~{Makaya}, T.~{He}, and
  K.~{Chan}, ``Adaptive federated learning in resource constrained edge
  computing systems,'' \emph{IEEE Journal on Selected Areas in Communications},
  vol.~37, no.~6, pp. 1205--1221, June 2019.

\bibitem{SGD_sparse1}
Y.~Lin, S.~Han, H.~Mao, Y.~Wang, and B.~Dally, ``Deep gradient compression:
  Reducing the communication bandwidth for distributed training,'' in
  \emph{International Conference on Learning Representations}, 2018.

\bibitem{SGD_q4}
H.~Tang, S.~Gan, C.~Zhang, T.~Zhang, and J.~Liu, ``Communication compression
  for decentralized training,'' in \emph{Advances in Neural Information
  Processing Systems 31}, S.~Bengio, H.~Wallach, H.~Larochelle, K.~Grauman,
  N.~Cesa-Bianchi, and R.~Garnett, Eds.\hskip 1em plus 0.5em minus 0.4em\relax
  Curran Associates, Inc., 2018, pp. 7662--7672.

\bibitem{fedlearn7}
\BIBentryALTinterwordspacing
F.~Sattler, S.~Wiedemann, K.~M{\"{u}}ller, and W.~Samek, ``Robust and
  communication-efficient federated learning from non-iid data,'' \emph{CoRR},
  vol. abs/1903.02891, 2019. [Online]. Available:
  \url{http://arxiv.org/abs/1903.02891}
\BIBentrySTDinterwordspacing

\bibitem{reuse}
V.~H.~M. Donald, ``Advanced mobile phone service: The cellular concept,''
  \emph{The Bell System Technical Journal}, vol.~58, no.~1, pp. 15--41, Jan
  1979.

\bibitem{5gfronthaul}
\BIBentryALTinterwordspacing
T.~Koelling and J.~Rogers, ``Exploring 5g fronthaul network architecture
  intelligence splits and connectivity,'' March 2019, White Paper. [Online].
  Available:
  \url{https://www.intel.co.uk/content/www/uk/en/architecture-and-technology/silicon-photonics/5g-fronthaul-network-architectures-paper.html}
\BIBentrySTDinterwordspacing

\bibitem{earth}
W.~Wajda, G.~Auer, P.~Skillermark, and Y.~Jading, ``Infso-ict-247733 earth
  deliverable d 2 . 3 energy efficiency analysis of the reference systems ,
  areas of improvements and target breakdown,'' 2012.

\bibitem{dataset2}
A.~Krizhevsky, ``Learning multiple layers of features from tiny images,'' Tech.
  Rep., 2009.

\bibitem{training2}
K.~{He}, X.~{Zhang}, S.~{Ren}, and J.~{Sun}, ``Deep residual learning for image
  recognition,'' in \emph{2016 IEEE Conference on Computer Vision and Pattern
  Recognition (CVPR)}, June 2016, pp. 770--778.

\bibitem{nn}
M.~S.~H. Abad, ``Hierarchical {FL},''
  \url{https://github.com/msalehihabad/hierarchicalFL}, 2019.

\bibitem{training1}
\BIBentryALTinterwordspacing
P.~Goyal, P.~Doll{\'{a}}r, R.~B. Girshick, P.~Noordhuis, L.~Wesolowski,
  A.~Kyrola, A.~Tulloch, Y.~Jia, and K.~He, ``Accurate, large minibatch {SGD:}
  training imagenet in 1 hour,'' \emph{CoRR}, vol. abs/1706.02677, 2017.
  [Online]. Available: \url{http://arxiv.org/abs/1706.02677}
\BIBentrySTDinterwordspacing

\bibitem{sgd_local3}
H.~Yu, R.~Jin, and S.~Yang, ``On the linear speedup analysis of communication
  efficient momentum {SGD} for distributed non-convex optimization,'' in
  \emph{Proceedings of the 36th International Conference on Machine Learning},
  ser. Proceedings of Machine Learning Research, K.~Chaudhuri and
  R.~Salakhutdinov, Eds., vol.~97.\hskip 1em plus 0.5em minus 0.4em\relax Long
  Beach, California, USA: PMLR, 09--15 Jun 2019, pp. 7184--7193.

\bibitem{fedlearn5}
\BIBentryALTinterwordspacing
Y.~Zhao, M.~Li, L.~Lai, N.~Suda, D.~Civin, and V.~Chandra, ``Federated learning
  with non-iid data,'' \emph{CoRR}, vol. abs/1806.00582, 2018. [Online].
  Available: \url{http://arxiv.org/abs/1806.00582}
\BIBentrySTDinterwordspacing

\bibitem{fedlearn9}
\BIBentryALTinterwordspacing
E.~Jeong, S.~Oh, H.~Kim, J.~Park, M.~Bennis, and S.~Kim,
  ``Communication-efficient on-device machine learning: Federated distillation
  and augmentation under non-iid private data,'' \emph{CoRR}, vol.
  abs/1811.11479, 2018. [Online]. Available:
  \url{http://arxiv.org/abs/1811.11479}
\BIBentrySTDinterwordspacing

\bibitem{training3}
\BIBentryALTinterwordspacing
C.~J. Shallue, J.~Lee, J.~M. Antognini, J.~Sohl{-}Dickstein, R.~Frostig, and
  G.~E. Dahl, ``Measuring the effects of data parallelism on neural network
  training,'' \emph{CoRR}, vol. abs/1811.03600, 2018. [Online]. Available:
  \url{http://arxiv.org/abs/1811.03600}
\BIBentrySTDinterwordspacing

\end{thebibliography}

\end{document}